\documentclass[lettersize,journal]{IEEEtran}
\usepackage{amsmath,amsfonts,amsthm,amssymb}
\usepackage{array}
\usepackage[caption=false,font=normalsize,labelfont=sf,textfont=sf]{subfig}
\usepackage{textcomp}
\usepackage{stfloats}
\usepackage{url}
\usepackage{booktabs}     
\usepackage{verbatim}
\usepackage{bm}
\usepackage{graphicx}
\usepackage{xcolor}
\usepackage{cite}
\usepackage{hyperref}
\usepackage{enumerate}
\usepackage{enumitem}
\usepackage[procnumbered,ruled,vlined,linesnumbered]{algorithm2e}

\def\eps{\epsilon}

\def\abs#1{\left|#1  \right|}

\def\trace#1{\mathrm{Tr} \left(#1 \right)}
\def\norm#1{\left\| #1 \right\|}

\newtheorem{theorem}{Theorem}[section]

\newtheorem{lemma}[theorem]{Lemma}

\newtheorem{remark}[theorem]{Remark}

\def\calG{\mathcal{G}}

\def\calL{\mathcal{L}}

\def\calH{\mathcal{H}}

\def\norm#1{\left\| #1 \right\|}

\newcommand{\removelatexerror}{\let\@latex@error\@gobble}

\newcommand{\walk}{\textsc{FastWalk}}
\newcommand{\chol}{\textsc{FastChol}}
\newcommand{\solver}{\textsc{AppSolver}}

\newcommand\LL{\bm{\mathit{L}}}

\newcommand{\midpar}[1]{\left[ #1 \right]}
\newcommand{\mean}[1]{\mathbb{E}\midpar{#1}}

\newcommand{\spar}[1]{\left( #1 \right)}

\def\trace#1{\mathrm{Tr} \left(#1 \right)}

\newcommand\ppi{\boldsymbol{\pi}}

\newcommand{\one}{\mathbf{1}}
\newcommand{\zero}{\mathbf{0}}

\newcommand\yy{\boldsymbol{\mathit{y}}}
\newcommand\zz{\boldsymbol{\mathit{z}}}
\newcommand\tautau{\boldsymbol{\mathit{\tau}}}

\newcommand\ttt{\boldsymbol{\mathit{t}}}
\newcommand\xx{\boldsymbol{\mathit{x}}}
\newcommand\ff{\boldsymbol{\mathit{f}}}

\newcommand\bb{\boldsymbol{\mathit{b}}}

\newcommand\dd{\boldsymbol{\mathit{d}}}
\newcommand\ee{\boldsymbol{\mathit{e}}}

\newcommand\vvv{\boldsymbol{\mathit{v}}}

\renewcommand\SS{\boldsymbol{\mathit{S}}}
\renewcommand\AA{\boldsymbol{\mathit{A}}}
\newcommand\BB{\boldsymbol{\mathit{B}}}

\newcommand\JJ{\boldsymbol{\mathit{J}}}
\newcommand\DD{\boldsymbol{\mathit{D}}}

\newcommand\EE{\boldsymbol{\mathit{E}}}
\newcommand\PP{\boldsymbol{\mathit{P}}}
\newcommand\MM{\boldsymbol{\mathit{M}}}

\newcommand\RR{\boldsymbol{\mathit{R}}}

\newcommand\II{\boldsymbol{\mathit{I}}}

\newcommand\ZZ{\boldsymbol{\mathit{Z}}}

\newcommand\ii{\boldsymbol{\mathit{i}}}

\DeclareMathOperator*{\argmax}{arg\,max}

\DontPrintSemicolon
\SetKw{KwAnd}{and}
\SetFuncSty{textsc}
\SetKwInOut{Input}{Input\ \ \ \ }
\SetKwInOut{Output}{Output}

\begin{document}

\title{Efficient Algorithms for Computing \\Random Walk Centrality}

\author{Changan~Liu,
        Zixuan~Xie,
        Ahad~N.~Zehmakan,
        and~Zhongzhi~Zhang,~\IEEEmembership{Member,~IEEE}\IEEEcompsocitemizethanks{
\IEEEcompsocthanksitem This work was supported by the National Natural Science Foundation of China (No. 62372112 and No. 61872093). \textit{(Corresponding author: Zhongzhi~Zhang.)} 

\IEEEcompsocthanksitem Changan~Liu, Zixuan~Xie, and Zhongzhi Zhang  are with Shanghai Key Laboratory of Intelligent Information Processing, College of Computer Science and Artificial Intelligence, Fudan University, Shanghai 200433, China.
\protect\\
E-mail: 19110240031@fudan.edu.cn, 20302010061@fudan.edu.cn,
zhangzz@fudan.edu.cn 
\IEEEcompsocthanksitem Ahad N. Zehmakan is with the School of Computing, the Australian National University, Canberra, Australia. 
\protect\\
E-mail: ahadn.zehmakan@anu.edu.au
}
\thanks{Manuscript received xxxx; revised xxxx.}
}

\markboth{IEEE TRANSACTIONS ON KNOWLEDGE AND DATA ENGINEERING, ~VOL. ~XX, NO. ~XX, MAR 2025}%
{Shell \MakeLowercase{\textit{et al.}}: Bare Demo of IEEEtran.cls for Computer Society Journals}


\maketitle

\begin{abstract}
Random walk centrality is a fundamental metric in graph mining for quantifying node importance and influence, defined as the weighted average of hitting times to a node from all other nodes. Despite its ability to capture rich graph structural information and its wide range of applications, computing this measure for large networks remains impractical due to the computational demands of existing methods. In this paper, we present a novel formulation of random walk centrality, underpinning two scalable algorithms: one leveraging approximate Cholesky factorization and sparse inverse estimation, while the other sampling rooted spanning trees. Both algorithms operate in near-linear time and provide strong approximation guarantees. 
Extensive experiments on large real-world networks, including one with over 10 million nodes, demonstrate the efficiency and approximation quality of the proposed algorithms.
\end{abstract}

\begin{IEEEkeywords}
Graph algorithm, Laplacian matrix,  random walk centrality, hitting time, Cholesky factorization, spanning tree.
\end{IEEEkeywords}

\section{Introduction}
\IEEEPARstart{C}{entrality} measures are fundamental concepts in network analysis, designed to quantify the importance and influence of nodes within a network~\cite{LU20161}. These measures find widespread applications in diverse fields, ranging from influence maximization~\cite{kempe2003maximizing} to the development of effective vaccination strategies~\cite{yang2019efficient}. Among the most popular centrality measures are degree, PageRank~\cite{Anatomy1998,PRbeyond}, closeness centrality~\cite{BaAl48}, and eigenvector centrality~\cite{BoPh72}. 

In recent years, random walk centrality (RWC) has emerged as a particularly powerful and versatile measure~\cite{NoJaRi04,lalo93,MaChMa15,ChHyHw23,ZhXuZh20,XiXuZh25}. Due to its ability to encode rich graph structural information at a global level, RWC has a better discriminating power~\cite{BaZh22} compared to many other centrality measures~\cite{ZhXuZh20,XiXuZh25}. This unique capability has led to RWC's adoption in a wide array of applications across various domains~\cite{LoMaKa07, BlFlTh11,JoBrKi19,OlStFu19,RiAlBo20,ChHyHw23}.

While RWC is crucial and widely applicable, its computation remains a significant challenge, particularly for large graphs. Exact computation relies on the spectrum of the normalized Laplacian matrix, which requires $O(n^3)$ time for a graph with $n$ nodes, making it impractical for large values of $n$. This computational bottleneck has motivated extensive research into more efficient algorithms for computing RWC in large graphs.

In response to this issue, 
Zhang et al.~\cite{ZhXuZh20} developed a novel randomized algorithm for approximating RWC. Their method reduces the estimation to approximating a quadratic form involving the pseudo-inverse of the Laplacian matrix, computed using a Laplacian solver~\cite{SpTe14,CoKyMiPaJaPeRaXu14,solver2023} that runs in linear time with respect to the number of edges. { However, this algorithm requires $O(\log n/\eps^2)$ invocations of the Laplacian solver for a parameter $\eps$ and the construction of a dense $n \times O(\log n/\eps^2)$ random matrix.} These requirements make it computationally intensive, rendering the approach impractical for large-scale networks with millions of nodes.

In this paper, we present more efficient and effective algorithms for calculating RWC, addressing the limitations of existing methods. Our approach is based on a key observation that computing RWC is equivalent to calculating the diagonal elements of the pseudo-inverse of the normalized Laplacian matrix~\cite{lalo93,Be16}. While exact computation of these elements traditionally requires $O(n^3)$ time, we introduce a novel formulation to accelerate this process.

Our formulation expresses the diagonal elements of the normalized Laplacian pseudo-inverse in terms of two components, namely the diagonal elements of the inverse of a submatrix of the normalized Laplacian, obtained by deleting a specific row and column, and the corresponding column elements of the normalized Laplacian pseudo-inverse. This decomposition forms the foundation for our more efficient algorithmic approach.

Building upon the new formulation, we develop two fast approximation methods. The first algorithm leverages the positive definiteness of the normalized Laplacian submatrix, enabling the use of Cholesky factorization~\cite{HePhSo20} for efficient computation of a solution. However, since calculating the exact Cholesky factorization and inversion of the factorization factors can be time-consuming, we employ incomplete Cholesky factorization and sparse Cholesky factor inversion. This approach exploits the fact that many elements in the inverse matrix are negligibly small and can be treated as zero, yielding a near-linear time algorithm with theoretical guarantees.

Our second algorithm capitalizes on the relationship between the diagonal elements of the inverse of normalized Laplacian submatrices and random walks with traps~\cite{ZhYa12}. We propose a Monte Carlo simulation method based on random spanning tree generation~\cite{Wi96}. Through novel theoretical analysis, we establish the required number of spanning tree samples and the corresponding approximation guarantee that can be achieved.

To validate our algorithms, extensive experiments on real-world networks ranging from several thousand to over 10 million nodes are conducted. The results demonstrate that our first algorithm achieves a significant efficiency improvement over the baseline, with only a slight compromise on approximation accuracy. Our second algorithm substantially outperforms the baseline in both efficiency and approximation quality. Specifically, both of our algorithms can effectively handle networks with over 3 million nodes, whereas the baseline fails to run due to memory overflow. Furthermore, for networks with over 10 million nodes, our first algorithm delivers approximate results with a mean relative error of only 3\% within approximately 4000 seconds. Meanwhile, our second algorithm provides an even more accurate approximation with a mean relative error of just 0.5\%, requiring slightly more computation time than the first algorithm but still significantly less than the baseline.

Our main contributions can be summarized as follows:
\begin{itemize}[leftmargin=*,topsep=0.5ex,partopsep=0ex]
    \item We introduce a novel formulation for calculating RWC { that constitutes the central theoretical contribution of this work}.
    \item This new formulation, { for the first time, shifts the main effort of estimating RWC to lightweight routines that estimate the diagonal of a normalized-Laplacian inverse, inspiring two algorithmic frames:} one leveraging the sparse inverse of the normalized Laplacian submatrix's incomplete Cholesky factor, and another exploiting its connection to random walks with traps on graphs.
    \item Through extensive experiments on various real-world and synthetic networks, we demonstrate that our proposed algorithms outperform existing methods in speed by several orders of magnitude, without compromising much on approximation quality.
\end{itemize}

\noindent \textbf{Roadmap.} In Section~\ref{sec:pre}, we form the ground for our study by providing the necessary definitions. In Section~\ref{sec:new-formulation}, a new formulation of RWC is established. Our two algorithms based on matrix factorization and sampling rooted spanning trees are given in Sections~\ref{sec:alg1} and~\ref{sec:alg2}. Section 6 presents our experimental results and their analysis.

\section{Preliminaries}\label{sec:pre}
\label{sec:pre}
This section sets up the stage for our study by introducing basic notations, concepts about graphs, and the existing algorithms.

\subsection{Notations}
Throughout this paper, we use the following notation unless otherwise specified. Lowercase letters (e.g., $a$) represent scalars, bold lowercase letters (e.g., $\xx$) denote vectors, and uppercase letters (e.g., $\MM$) indicate matrices. For indexing elements, we use subscripts. Specifically, $\xx_i$ represents the $i$-th element of vector $\xx$, while $\MM_{ij}$ denotes the element at the $i$-th row and $j$-th column of matrix $\MM$. Additionally, $\MM_{i,:}$ and $\MM_{:,j}$ represent the $i$-th row and $j$-th column of $\MM$, respectively. We use $\ee_i$ for the $i$-th standard basis vector and $\xx^{\top}$ for the transpose of vector $\xx$. The symbol $\zero$ denotes the all-zeros vector, whereas $\JJ$ represents the all-ones matrix. Finally, $\MM_{i}$ stands for the submatrix of $\MM$ obtained by removing its $i$-th row and column. For any vector $\xx$, $\xx_{-i}$ denotes its subvector obtained by removing the $i$-th element.

\subsection{Graph and Corresponding Matrices}
Consider a connected undirected graph (network) $\calG = (V, E)$ with nodes $V$ and edges $E \subseteq V \times V$. Let $n := |V|$ and $m := |E|$ denote the number of nodes and the number of edges, respectively. A rooted spanning tree is a connected subgraph of $\calG$ which has 
$n$ nodes and $n-1$ edges with one node designated as the root. We use $\mathcal{N}\left(i\right)$ to denote the set of neighbors of $i$, and the degree of node $i$ is $\dd_{i}=\left|\mathcal{N}\left(i\right)\right|$. The Laplacian matrix of $\calG$ is the symmetric matrix $\LL = \DD - \AA$, where $\AA\in\{0,1\}^{n\times n}$ is the adjacency matrix whose entry $\AA_{ij}=1$ if node $i$ and node $j$ are adjacent, and $\AA_{ij}=0$ otherwise.
The matrix $\DD=\text{diag}(\dd_1,\cdots,\dd_n)$ represents the diagonal degree matrix and $d_{\rm{max}}$ denotes the maximum degree of nodes.

For any pair of distinct nodes $u,v\in V$, we define $\bb_{uv} = \ee_{u}-\ee_{v}$. We fix an arbitrary orientation for all edges in $\calG$, then we can define the signed edge-node incidence matrix $\BB^{m\times n}$ of graph $\calG$, whose entries are defined as follows: $\BB_{eu}= 1$ if node $u$ is the head of edge $e$, $\BB_{eu}= -1$ if $u$ is the tail of $e$, and $\BB_{eu}= 0$ otherwise. $\LL$ can be rewritten as $\LL = \BB^\top\BB$ and it is positive semi-definite with its Moore-Penrose pseudo-inverse being $\LL^\dagger = \big(\LL +\frac{1}{n}\JJ\big)^{-1}-\frac{1}{n}\JJ$. Finally, we use $\calL=\DD^{-1/2}\LL\DD^{-1/2}$ to denote the normalized Laplacian matrix.

\subsection{Random Walk Centrality}
In a random walk on a graph $\calG$, at any discrete-time step, the walker moves from its current node $i$ to node $j$ independently with probability $\AA_{ij}/\dd_i$ (i.e., it moves to one of the neighboring nodes with a uniform probability). This process can be characterized by the transition matrix $\PP$, where the element $\PP_{ij}$ is equal to $\AA_{ij}/\dd_i$. If $\calG$ is finite and non-bipartite, the random walk has the following unique stationary distribution (cf.~\cite{LiZh13PRE}):
\begin{equation}\label{EE01}
\ppi=(\pi_1, \pi_2, \cdots, \pi_n)^{\top}=\left(\frac{\dd_1}{2m}, \frac{\dd_2}{2m}, \cdots, \frac{\dd_n}{2m}\right)^{\top}.
\end{equation}

A fundamental quantity for random walks is \textit{hitting time}~\cite{lalo93,CoBeTeVoKl07}. The hitting time $\calH_{ij}$ represents the expected number of steps for a walker starting at node $i$ to reach node $j$ for the first time.

The \textit{random walk centrality}~\cite{MaChMa15} $\calH_u$ of a node $u$, defined as $\calH_u=\sum_{i} \rho(i) \calH_{iu}$, where $\rho(\cdot)$ is the starting probability distribution over all nodes in $V$. By definition, $\calH_u$ is a weighted average of hitting times to node $u$. Nodes with smaller values of $\calH_u$ are considered more central. Unlike shortest-path based centrality measures, RWC accounts for contributions from all paths~\cite{Ne05}, offering a richer description of a node's importance. 

In our study, we consider $\rho(\cdot)$ as the stationary distribution $\ppi$, simplifying $\calH_u$ to $\sum_{i} \pi_i \calH_{iu}$. This formulation of $\calH_u$, which has been extensively studied~\cite{TeBeVo09,Be09,Be16}, measures the expected number of steps to reach node $u$ when starting from a random node following the stationary distribution. According to~\cite{lalo93,Be16}, we have
\begin{align}\label{eq:random_walk_centrality}
    \calH_u=\sum_{i=1}^n \pi_i \calH_{i u}=\frac{(\calL^{\dagger})_{uu}}{\pi_u},
\end{align}
where $\calL^{\dagger}$ is the pseudo-inverse of $\calL$. Thus, $\calH_u$ can be computed using only the diagonal of $\calL^{\dagger}$.

\subsection{Existing Methods}
Direct computation of RWC involves inverting the normalized Laplacian matrix (see Eq.~\eqref{eq:random_walk_centrality}) which runs in $O(n^3)$ time. Thus, prior work has resorted to approximate solutions. The state-of-the-art algorithm is proposed by~\cite{ZhXuZh20}, where the authors established a connection between RWC and the quadratic form of the Moore-Penrose pseudo-inverse of the Laplacian matrix $\LL^\dagger$. They effectively leveraged the Johnson-Lindenstrauss lemma and Laplacian solvers to compute this centrality measure. Their algorithm runs in $\Tilde{O}(m/\eps^2)$ time where $\Tilde{O}(\cdot)$ hides $\text{poly}(\log n)$ factors, marking a significant advancement over the previous algorithms.

However, this approach faces several critical challenges when applied to large-scale networks. Primarily, the memory requirements of the Laplacian solver are substantially high. Furthermore, the construction of an {$n \times
 O(\log n/\epsilon^2)$} random matrix, a prerequisite for their method, imposes additional memory constraints that can be prohibitive in large-scale networks. More importantly, the algorithm requires {$O(\log n/\epsilon^2)$} executions of the Laplacian solver. These repeated calls to the solver introduce significant computational overhead, especially as the network size grows. 
These limitations highlight the need for more scalable methods that preserve accuracy while reducing computational complexity, especially for the massive networks common in modern applications.

\section{New formulation of RWC}
\label{sec:new-formulation}
The primary computational challenge in the algorithm of  Zhang et al.~\cite{ZhXuZh20} stems from its repeated calls to the Laplacian solver. This prompts a natural question: can we reduce the number of solver calls? To this end, we next present a novel formulation for RWC. 

\begin{theorem}\label{the:new_formular}
Let $v\in V$ be a designated pivot node, $\calL_v$ be the submatrix of the normalized Laplacian matrix $\calL$ obtained by deleting the $v$-th row and the $v$-th column of $\calL$, then we have the following reformulation of RWC for all nodes $u\in V$:
\begin{align}\label{eq:newformular}
      \calH_u = \frac{1}{\pi_u}\Big((\calL_v^{-1})_{uu} - \frac{\dd_u}{\dd_v}(\calL^\dagger)_{vv} + 2\frac{\sqrt{\dd_u}}{\sqrt{\dd_v}}(\calL^\dagger)_{uv}\Big).
\end{align}
\end{theorem}
\begin{proof}
For the pivot node $v$, setting $(\calL_v^{-1})_{vv}=0$ reduces Eq.~\eqref{eq:newformular} to Eq.~\eqref{eq:random_walk_centrality}, which holds trivially.

For other nodes, we envision the network as an electrical circuit where every edge in \( E \) acts as a unit resistor, and the nodes in \( V \) function as junctions connecting these resistors. When a unit current is injected into node \( u \) and extracted from node \( v \), the resulting current \( \ii \) on the edges adheres to Kirchhoff's current law~\cite{DoSn84}, which asserts that the total current flowing into node \( v \) must equal the current injected into it:
$
\bb_{uv}^\top = \ii^\top \BB.
$

For any potential vector $\ff$, according to Ohm's law~\cite{DoSn84}, the induced current on the edge from \( u \) to \( v \) is given by the product of \( \ff_u - \ff_v \) and the edge conductance:
$
\ii^\top = \ff^\top\BB^\top.
$
Therefore, we have
$$
\bb_{uv}^\top \DD^{-1 / 2} = \ff^\top\BB^\top  \BB \DD^{-1 / 2} = \ff \DD^{1 / 2} \calL.
$$
Assume \( \ff \) satisfies \( \sum_v \ff_v = 0 \). By Green's function~\cite{LiZh13IEEE,ThomsonAnEO1850}, 
$$
\bb_{uv}^\top \DD^{-1 / 2} \calL^\dagger \DD^{-1 / 2} = \ff^\top.
$$
The effective resistance~\cite{Te91} between \( u \) and \( v \) is thus
$$
\begin{aligned}
R_{uv} = \ff^\top \bb_{uv} 
= \bb_{uv}^\top \DD^{-1 / 2} \calL^\dagger \DD^{-1 / 2} \bb_{uv}.
\end{aligned}
$$
Therefore, we have
{
$$
\begin{aligned}
&(\DD^{-1 / 2} \calL^\dagger \DD^{-1 / 2})_{uu} \\
=& R_{uv} - (\DD^{-1 / 2} \mathcal{L}^\dagger \DD^{-1 / 2})_{vv} + 2(\DD^{-1 / 2} \mathcal{L}^\dagger \DD^{-1 / 2})_{uv}.
\end{aligned}
$$}
By multiplying both sides by \( \dd_u \), we get
\begin{align}\label{eq:er_diag}
(\mathcal{L}^\dagger)_{uu} = \dd_u R_{uv} - \frac{\dd_u}{\dd_v} (\mathcal{L}^\dagger)_{vv} + 2 \frac{\sqrt{\dd_u}}{\sqrt{\dd_v}} (\mathcal{L}^\dagger)_{uv}.
\end{align}
According to~\cite{INK+13}, the resistance distance between nodes \( u \) and \( v \) is \( (\LL_{v}^{-1})_{uu} \). Considering that
\[
(\LL_v^{-1})_{uu}=((\II - \PP_v)^{-1}\DD_v^{-1})_{uu} = \frac{(\II-\PP_v)_{uu}^{-1}}{\dd_u},\quad \text{and}\] 
\[ (\mathcal{L}_{v}^{-1})_{uu} = (\DD_v^{-1/2}(\II - \PP_v)^{-1} \DD_v^{1/2})_{uu} = (\II - \PP_{v})^{-1}_{uu} ,
\]thus $ \dd_u R_{uv}=(\calL_v^{-1})_{uu}$. Plugging it into Eq.~\eqref{eq:er_diag} ends the proof.
\end{proof}
{
\begin{remark}[RWC versus effective resistance~\cite{Te91}]
The quantity $\calH_u$ in Eq.~\eqref{eq:newformular} is \textit{not} equal to the classical effective resistance.  The difference is two-fold.  First, RWC is built on the \textit{normalised} Laplacian~$\calL$, whereas effective resistance is defined with the Laplacian~$\LL$.  Second, effective resistance is inherently a \textit{pairwise} notion, depending on a chosen node pair while $\calH_u$ serves as a \textit{global} centrality that describes the position of node $u$ within the whole network, irrespective of which pivot node $v$ is chosen. Thus, the two quantities serve complementary analytic roles.
\end{remark}}

Building on Theorem~\ref{the:new_formular}, we propose a novel paradigm for computing RWC.

\begin{center}
\renewcommand{\arraystretch}{1.1}
\begin{tabular}{p{0.95\linewidth}}
\toprule
\textbf{\textsc{New computation paradigm for RWC}}  \\ \midrule
\begin{enumerate}[label*=\arabic*. , leftmargin=1.5em]
    \item \textbf{Pivot solve.}  Choose a pivot $v\in V$ and solve the linear system
          $\calL\yy=\ee_v$ once.  The solution
          $\yy=\calL^{\dagger}\ee_v$ is exactly the $v$-th column of $\calL^{\dagger}$.
    \item \textbf{Diagonal estimation.}  Approximate the remaining $n-1$ diagonal entries of
          $\calL^{-1}_{v}$ without further global solves.
    \item \textbf{Centrality recovery.}  Compute the random-walk centrality
          $\calH_{u}$ for every $u\in V$ via Eq.~\eqref{eq:newformular}.
\end{enumerate} \\ \bottomrule
\end{tabular}
\end{center}

For the first step, which involves solving the linear system to obtain $\calL^\dagger_{:, v}$, an exact solution is time-consuming as it involves matrix inversion. To accelerate the computation, we turn to the following linear Laplacian solver. 
\begin{lemma} \label{lem:solver}(Fast SDDM Solver~\cite{SpTe14,CoKyMiPaJaPeRaXu14,solver2023}) There is a nearly linear-time solver $\mathbf{g} = \texttt{LapSolve}(\LL, \bb, \theta)$ which takes a symmetric positive semi-definite matrix $\LL$ with $m$ nonzero entries, a vector $\bb \in \mathbb{R}^n$, and an error parameter $\theta > 0$, and returns a vector $\mathbf{g} \in \mathbb{R}^n$ satisfying $\norm{\mathbf{g} - \LL^{-1} \bb}_{\LL} \leq \theta \norm{\LL^{-1} \bb}_{\LL}$ with probability at least $1-1/n$, where $\norm{\mathbf{g}}_{\LL} = \sqrt{\mathbf{g}^\top \LL \mathbf{g}}$. This solver runs in expected time $\tilde{O}(m)$, where $\tilde{O}(\cdot)$ notation suppresses the ${\rm poly} (\log n)$ factors. 
\end{lemma}
However, the off-the-shelf solver accepts only the unnormalised Laplacian $\LL$, so it cannot be invoked on the normalised form $\calL$ directly.  Adopting the strategy of~\cite{ZhXuZh20}, we bridge this gap by exploiting the identity below, which relates the Moore–Penrose pseudoinverses of $\LL$ and $\calL$~\cite{bozzo2013moore}:

\begin{align}
    \calL^\dagger = \Big(I - \frac{1}{2m}\DD^{\frac{1}{2}}\one\one^\top\DD^{\frac{1}{2}}\Big)\DD^{\frac{1}{2}}\LL^\dagger\DD^{\frac{1}{2}}\Big(I - \frac{1}{2m}\DD^{\frac{1}{2}}\one\one^\top\DD^{\frac{1}{2}}\Big).
\end{align}

Therefore, we can obtain an approximation $\tilde{\yy}$ for $\calL_{:, v}^{\dag}$ by calling $\texttt{LapSolve}(\LL, \bb, \theta)$ where $\bb = \DD^{1/2}\left(I - \frac{1}{2m}\DD^{\frac{1}{2}}\one\one^\top\DD^{\frac{1}{2}}\right)\ee_v$ and then performing a simple matrix-vector multiplication. Assuming we have the (approximate) diagonal elements of $\calL_v^{-1}$, we can efficiently calculate the RWCs according to Eq.~\eqref{eq:newformular}. The complete procedure for estimating RWC $\tilde{\calH}_u$ is presented in Algorithm~\ref{alg:apprwc}.


\begin{algorithm}[h]
    \caption{$\textsc{AppRWC}(\calG, v, \theta, \{(\calL_v^{-1})_{uu}|u \in V \setminus \{v\}\}$)}\label{alg:apprwc}
    \Input{
        A connected graph $\calG=(V,E)$ with $n$ nodes; 
        \\pivot node $v$; 
        error parameter $\theta$;
        \\$(\calL_v^{-1})_{uu}$ for all $u \in V \setminus \{v\}$
    }
    \Output{
        Approximated RWC $\tilde{\calH}_u$ for all $u \in V$
    }
    $\bb = \DD^{1/2}\left(I - \frac{1}{2m}\DD^{\frac{1}{2}}\one\one^\top\DD^{\frac{1}{2}}\right)\ee_v$\;
    $\zz = \texttt{LapSolve}(\LL, \bb, \theta)$\;
    $\tilde{\yy} = \left(I - \frac{1}{2m}\DD^{\frac{1}{2}}\one\one^\top\DD^{\frac{1}{2}}\right)\DD^{\frac{1}{2}}\zz$\;
    \For{$u \in V \setminus \{v\}$}{
        $\tilde{\calH}_u = \frac{1}{\pi_u}\left((\calL_v^{-1})_{uu} - \frac{\dd_u}{\dd_v}\tilde{\yy}_v + 2\frac{\sqrt{\dd_u}}{\sqrt{\dd_v}}\tilde{\yy}_u\right)$\;
    }
    $\tilde{\calH}_v = \frac{1}{\pi_v}\tilde{\yy}_v$\;
    \Return $\tilde{\calH}_u$ for all $u \in V$\;
\end{algorithm}

{To summarize this section, Theorem~\ref{the:new_formular} inspires a novel computation paradigm of the computation of RWC including three lightweight steps: a \emph{single} Laplacian solve for a pivot column of $\calL^{\dagger}$, followed by the estimation of $n-1$ diagonal entries of $\calL_v^{-1}$ and a final inexpensive aggregation.  
In contrast, the best existing method~\cite{ZhXuZh20} invokes a Laplacian solver \(O\bigl(\log n/\epsilon^{2}\bigr)\) times, so both running time and memory scale with that factor.  
By collapsing the sequence of linear solves to a \emph{single} solve for the pivot column, the new formulation eliminates the dominant cost term and \emph{shifts} all remaining effort to lightweight routines that estimate the diagonal of \(\calL_v^{-1}\).  Two complementary instantiations of this idea follow: Section~\ref{sec:alg1} exploits the symmetric positive-definiteness of \(\calL_v\) to build a sparse incomplete-Cholesky surrogate, while Section~\ref{sec:alg2} achieves the same goal through a rooted spanning-tree sampler. Together, they compute the paradigm in near-linear time and significantly cut the overall memory and runtime.}

\section{$\chol$ Algorithm: Matrix Factorization}
\label{sec:alg1}
{ Building on the reformulation in Section~\ref{sec:new-formulation}, which reduces RWC to accessing diagonal entries of $\calL_v^{-1}$, and taking advantage of the symmetric positive definiteness of $\calL_v$ established in~\cite{JJMCNe95}, we next approximate those entries by replacing the exact inverse with a sparse substitute obtained from an incomplete-Cholesky factor.
}

\subsection{Theoretical Foundation}

To overcome the computational complexity of direct matrix inversion, we propose a two-step approach that significantly reduces time complexity, consisting of incomplete Cholesky factorization and sparse approximate inverse.

The Cholesky factorization, widely used for its efficiency in handling symmetric positive definite matrices~\cite{DaTiRa16,BeMi02,HePhSo20,liu2023computing}, produces a lower triangular matrix $\ZZ$ such that $\AA = \ZZ\ZZ^\top$. This factorization is advantageous as obtaining $\ZZ^{-1}$ can significantly reduce the computational cost compared to directly inverting $\AA$.

Given the computational challenges associated with full matrix factorization, particularly for large-scale systems, incomplete Cholesky factorization is widely used instead~\cite{Saad94,LinMor99}. This approach approximates the full factorization while maintaining sparsity, which proceeds similarly to the full Cholesky factorization but drops fill-in elements below a certain threshold, often guided by the original matrix's sparsity pattern~\cite{Meijerink77}. As $\calL_v$ is symmetric positive definite, we can apply this technique to efficiently approximate $\calL_v^{-1}$. The following lemma  formalizes the properties of this incomplete factorization.
\begin{lemma}
\label{lem:ichol_error_bound}
Let $\calG = (V, E)$ be a connected graph and let $\calL_v$ be its corresponding normalized Laplacian submatrix. Given parameter $\delta > 0$ (drop threshold), there exists an algorithm \textsc{IChol}($\mathcal{L}_v, \delta$) that runs in $O(m + nnz(\RR))$ time ($nnz(\RR)$ is the number of non-zero elements in $\RR$) and returns a lower triangular matrix $\RR \in \mathbb{R}^{(n-1) \times (n-1)}$. The relative error in Frobenius norm is bounded by
\begin{equation}
\frac{\|\calL_v - \RR\RR^\top\|_F}{\|\calL_v\|_F} = O(\sqrt{k} \cdot \delta ),
\end{equation}
where $k$ is the number of elements dropped during factorization, and for a square matrix $\MM$, $\|\MM\|_F = \sum_i\sum_j \MM_{ij}^2$. 
\end{lemma}
\begin{proof}
    Let $\EE = \mathcal{L}_v - \RR\RR^\top$ be the error matrix. For the incomplete Cholesky factorization, off-diagonal elements $(\mathcal{L}_v)_{ij}$ are dropped if 
\begin{equation*}
    |(\mathcal{L}_v)_{ij}| < \delta \cdot \sqrt{|(\mathcal{L}_v)_{ii}| \cdot |(\mathcal{L}_v)_{jj}|}.
\end{equation*}
Therefore, each dropped element contributes at most $\delta^2 \cdot |(\mathcal{L}_v)_{ii}| \cdot |(\mathcal{L}_v)_{jj}|$ to $\|E\|_F^2$.

Since $\mathcal{L}_v$ is derived from a normalized Laplacian matrix, we know that its diagonal elements are equal to 1. The total number of dropped elements is $k$, so we have
\begin{equation*}
\|\EE\|_F^2 \leq k \cdot \delta^2 \cdot \max_{ij} (|(\mathcal{L}_v)_{ii}| \cdot |(\mathcal{L}_v)_{jj}|) \leq k \cdot \delta^2.
\end{equation*}

Taking the square root and dividing by $\|\mathcal{L}_v\|_F$, we have
\begin{equation}
\frac{\|\EE\|_F}{\|\mathcal{L}_v\|_F} \leq \frac{\sqrt{k} \cdot \delta}{\|\mathcal{L}_v\|_F} = O(\sqrt{k} \cdot \delta).
\end{equation}

The last equality holds because $\|\mathcal{L}_v\|_F \geq 1$ for a normalized Laplacian matrix with at least one non-zero off-diagonal element.

The time complexity $O(m + nnz(\RR))$ follows from the fact that the algorithm needs to process all $m$ edges of the original graph and perform operations on the non-zero elements of $\RR$.
\end{proof}

With the incomplete Cholesky factorization, we can now approximate $\calL_v$ as $\RR\RR^\top$. This allows us to reformulate computing the diagonal entries of $\calL_v^{-1}$ in terms of $\RR$:
\begin{equation}\label{norm}
(\calL_v^{-1})_{uu} \approx \ee_u^\top (\RR^{-1})^\top\RR^{-1}\ee_u = \|\RR^{-1}\ee_u\|_2^2.
\end{equation}
However, computing $\RR^{-1}$ is costly. This brings us to the second step of our strategy: devising an efficient method to approximate $\RR^{-1}\ee_u$ for all $u\in V\setminus\{v\}$ without explicitly computing  $\RR^{-1}$.

\subsection{Algorithm Design}
The core idea is to construct a sparse matrix $\widetilde{\SS}$ that approximates $\RR^{-1}$.
We now turn our attention to the core of our sparse approximation strategy. The efficiency of our approach hinges on a crucial property of the inverse of the incomplete Cholesky factor, which is based on the key observation given in the following lemma.

\begin{lemma}\label{inversedef}
Let $\RR$ be the incomplete Cholesky factor of $\calL_v$, and let $\SS = \RR^{-1}$, $\SS$ is then non-negative. For any $u$-th column of $\SS$, 
\begin{equation}\label{eq:zj}
\SS_{:,u} = \frac{1}{\RR_{uu}}\ee_u - \frac{1}{\RR_{uu}}\sum_{i=u+1}^n \RR_{iu}\SS_{:,i}.
\end{equation}
\end{lemma}
\begin{proof}
    \textbf{Verification of Eq.~\eqref{eq:zj}:}
Multiply both sides of the equation by $\RR$ and for the right-hand side,
\begin{align*}
&\quad \frac{1}{\RR_{uu}}\RR\ee_u - \frac{1}{\RR_{uu}}\sum_{i=u+1}^n \RR_{iu}\RR\SS_{:,i} \\
&= \frac{1}{\RR_{uu}}((0, \ldots, 0, \RR_{uu}, \RR_{u+1,u}, \ldots, \RR_{nu})^\top - \sum_{i=u+1}^n \RR_{iu}\ee_i)\\
&= \ee_u + \vvv - \frac{1}{\RR_{uu}}\sum_{i=u+1}^n \RR_{iu}\ee_i = \ee_u + \vvv - \vvv = \ee_u,
\end{align*}
\text{where } $\vvv = (0, \ldots, 0, 0, \frac{\RR_{u+1,u}}{\RR_{uu}}, \ldots, \frac{\RR_{nu}}{\RR_{uu}})^\top $.

For the left-hand side, with the definition of $\SS$, we know $\RR\SS_{:,u} = \ee_u$, verifying Eq.~\eqref{eq:zj}.

\textbf{Non-negativity of $\SS$:}
We use induction on $u$, starting from $n$ and moving backwards.

Base case ($u = n$):
$\SS_{:,n} = \frac{1}{\RR_{nn}}\ee_n \geq 0$, as $\RR_{nn} > 0$.

Inductive step:
Assume $\SS_{:,i} \geq 0$ for all $i > u$. From Eq.~\eqref{eq:zj}, we have
\begin{equation*}
\SS_{:,u} = \frac{1}{\RR_{uu}}\ee_u - \frac{1}{\RR_{uu}}\sum_{i=u+1}^n \RR_{iu}\SS_{:,i}.
\end{equation*}
Since $\RR_{uu} > 0$, $\RR_{iu} \leq 0$ for $i > u$, and $\SS_{:,i}$ is non-negative for $i > u$ by induction hypothesis, we have that $\SS_{:,u}$ is non-negative for all $u$ and thus $\SS$ is non-negative.
\end{proof}

Based on Lemma 4.3, a right-to-left recursive approach is suggested for computing the columns of $\SS$. 
That is, to compute the $u$-th column of $\SS$, we require information from all subsequent columns $i$ where $\RR_{iu} \neq 0$. Nevertheless, this approach incurs significant computational cost, with a time complexity of $O(n^2)$ owing to the dense nature of $\SS$. In our context, we propose two strategies to reduce this complexity.

\begin{algorithm}[h]
\caption{\chol($\mathcal{G}, v, \epsilon_{p}, \delta, w_s, \zeta, \theta$)}
\label{alg:SparseChol}
\KwIn{A connected graph $\mathcal{G} = (V,E)$; threshold $\epsilon_{p}$; pivot node $v$; parameters $\delta$ for \textsc{IChol}; base window size $w_s$; error parameter $\theta$}
\KwOut{$\tilde{\calH}_u~ \forall u \in V$}
$\RR \gets \textsc{IChol}(\calL_v, \delta)$\;
Set $\widetilde{\SS}$ as an $(n-1) \times (n-1)$ zero matrix and $\tilde{\tautau} =\zero^{(n-1)\times1}$\;
$end \gets n-1$\;
\For{$u = n-1$ \KwTo $1$}{
    $ w_s \gets \min (w_s \cdot (1 + \dd_u / d_{\max}), n)$\;
     \lIf{$end - (u+1) > w_s$}{
    $end \gets u + w_s$
    }
        Initialize $\SS^*_{:,u}  \gets \zero^{(n-1)\times 1}$\;
    $\SS^*_{uu} \gets 1/\RR_{uu}$\;
    \For{$i = u+1$ \KwTo $u+w_s$}{
    \lIf{$\RR_{iu} \neq 0$}{
    $\SS^*_{:,u} \gets \SS^*_{:,u} - \RR_{iu}/\RR_{uu} \cdot \widetilde{\SS}_{:,i-1}$
    }
    }
    \lIf{$\text{nnz}(\SS^*_{:,u}) \leq \zeta$}{
        \Return $\SS^*_{:,u}$
    }
    $\SS^{**}_{:,u} \gets \text{Sort}(|\SS^*_{:,u}|)$\tcp{elementwise absolution}
    $k^* \gets \argmax_{k}{\frac{\|\SS^{**}_{k+1:end,u}\|_1}{\|\SS^*_{:,u}\|_1} \leq \epsilon_{p}}$\;
    Discard elements in $\SS^*_{:,u}$ below $(\SS^{**}_{:,u})_{k^*}$ to obtain $\widetilde{\SS}_{:,u}$\;   
    $\tilde{\tau}_u \gets \|\widetilde{\SS}_{:,u}\|_2^2$\;
}
$\tilde{\calH}_u \gets$ \textsc{AppRWC}$(\mathcal{G}, v, \theta, \tilde{\tau})$\;
\Return {$\tilde{\calH}_u\ \forall u\in V$ \;}
\end{algorithm}

First, we observe that the elements in $\SS$ are generally small. Thus, we can discard the insignificant elements in each column while focusing our attention on more significant ones, a process we refer to as sparsification. The sparsification process uses $\zeta$ to determine if discarding is necessary. If the number of non-zero elements in the initial estimate $\SS^*_{:,u}$ exceeds $\zeta$, we then use a threshold  $\epsilon_p$ to decide which elements to discard. Specifically, only the $k$ largest elements with an approximation error not exceeding $\epsilon_{p}$ are retained. This approach ensures the sparsity upper bound $\zeta$ is respected while maintaining the most accurate sparse approximation within the error tolerance $\epsilon_{p}$.

Next, to further reduce complexity, we introduce a \emph{sliding window} technique that limits the calculation of each column to only a few nearby columns on the right. This approach leverages the fact that these nearby columns already incorporate information from further right columns. Recognizing the scale-free nature of real-world networks, we adjust the window size $w_s$ based on node degree, scaling it with the ratio of the current node's degree to the maximum degree. This adaptive window allows higher-degree nodes to benefit from more information. Using this technique, we construct a final sparse approximation $\widetilde{\SS}_{:,u}$ for each column.

Incorporating these ideas, we introduce $\chol$ for estimating RWC, which is outlined in Algorithm~\ref{alg:SparseChol}. It begins with incomplete Cholesky factorization of $\calL_v$ to obtain $\RR$. The core of the algorithm is a loop that computes each column of $\widetilde{\SS}$ from right to left, employing two key techniques for each column. First, the sliding window technique is implemented through adaptive window sizing (Lines 5-6), which limits calculations to nearby columns based on node degree. This is followed by the sparsification process (Lines 11-14), where the elements in $\SS^*$ are sorted in descending order to facilitate efficient sparsification. After computing the sparse approximations of $\SS_{:,u}$, we obtain the diagonal elements $\tilde{\tau}_u$ according to Eq.~\eqref{norm}. Finally, the algorithm uses \textsc{AppRWC} to estimate RWCs. 

\subsection{Performance Analysis}
We now present a performance analysis of our algorithm $\chol$, which is summarized by the following theorem.
\begin{theorem}
\label{thm:complexity_and_error}
Given a connected graph $\calG = (V, E)$, $\chol$ runs in $O(n(w_s \log n + \log^2 n)) + \tilde{O}(m)$ time. For any node $u \in V$, let $\widetilde{\SS}_{:,u}$ be the approximation of $\SS_{:,u}=\RR^{-1}\ee_u$ (Line 14 in Algorithm~\ref{alg:SparseChol}), the relative error between them can be bounded by
\begin{equation}\label{eq:bound_chol}
\frac{|\|\SS_{:,u}\|_2^2 - \|\widetilde{\SS}_{:,u}\|_2^2|}{\|\SS_{:,u}\|_2^2} \leq 2\epsilon_{wd} - \epsilon_{wd}^2 + \epsilon_{p}^2 n (1 + \epsilon_{wd})^2 = \epsilon,
\end{equation}
where $w_s$ is the base window size, $\epsilon_{p}$ is the sparsification parameter, and $\epsilon_{wd}$ is the error introduced by the sliding window.
\end{theorem}

\begin{proof}
According to Lemma~\ref{lem:ichol_error_bound}, incomplete Cholesky factorization takes $O(m + nnz(\RR))$ time. For sparse graphs, both $m$ and $nnz(\RR)$ are typically $O(n)$. The main loop iterates $n$ times, with each iteration involving $O(w_s \log n + \log^2 n)$ for sparsifying the corresponding columns and $O(\log n)$ for maintenance of the sliding window. The final computation of $\tilde{\calH}_u$ for all $u \in V$ takes $\tilde{O}(m)$. Summing these components yields the stated time complexity.

Let $\SS_{:,u}$ be the exact $u$-th column of $\RR^{-1}$, $\SS_{:,u}^*$ be the approximation after utilizing the sliding window technique, and $\widetilde{\SS}_{:,u}$ be the final sparsified result of $\SS_{:,u}^*$.

For the sparsification error, our algorithm ensures
\begin{equation*}
\frac{\|\SS_{:,u}^* - \widetilde{\SS}_{:,u}\|_1}{\|\SS_{:,u}^*\|_1} \leq \epsilon_{p}.
\end{equation*}

Converting this 1-norm bound to 2-norm, we have
\begin{equation*}
\|\SS_{:,u}^* - \widetilde{\SS}_{:,u}\|_2^2 \leq \|\SS_{:,u}^* - \widetilde{\SS}_{:,u}\|_1^2 \leq (\epsilon_{p}\|\SS_{:,u}^*\|_1)^2 \leq \epsilon_{p}^2 n \|\SS_{:,u}^*\|_2^2.
\end{equation*}

Therefore,
\begin{equation}
\notag\frac{|\|\SS_{:,u}^*\|_2^2 - \|\widetilde{\SS}_{:,u}\|_2^2|}{\|\SS_{:,u}^*\|_2^2} \leq \epsilon_{p}^2 n.
\end{equation}

Let $\epsilon_{wd}$ be the relative error introduced by the sliding window, then we have
\begin{equation}\notag
\frac{\|\SS_{:,u} - \SS_{:,u}^*\|_2}{\|\SS_{:,u}\|_2} \leq \epsilon_{wd}.
\end{equation}

The sparsification error is bounded as before. We can relate $\|\SS_{:,u}^*\|_2$ to $\|\SS_{:,u}\|_2$, then we obtain
\begin{equation*}
\|\SS_{:,u}^*\|_2 \leq \|\SS_{:,u}\|_2 + \|\SS_{:,u} - \SS_{:,u}^*\|_2 \leq \|\SS_{:,u}\|_2 (1 + \epsilon_{wd}).
\end{equation*}

By combining these inequalities, we get
\begin{align}
\|\SS_{:,u}\|_2^2 - \|\widetilde{\SS}_{:,u}\|_2^2 
&
\leq \|\SS_{:,u}\|_2^2 - \|\SS_{:,u}^*\|_2^2 + \|\SS_{:,u}^* - \widetilde{\SS}_{:,u}\|_2^2\notag \\
&
\leq \|\SS_{:,u}\|_2^2 (2\epsilon_{wd} - \epsilon_{wd}^2 + \epsilon_{p}^2 n (1 + \epsilon_{wd})^2)\notag.
\end{align}

Applying the triangle inequality to the error terms allows us to derive the upper bound stated in the theorem.

The $\epsilon_{wd}$ error arises from using only the most recent vectors stored in $\widetilde{\SS}$ to compute $\SS_{:,u}^*$. For the $j$-th column, we can obtain
\begin{equation*}
\SS_{:,j} = \frac{\mathbf{e}_j}{\RR_{jj}} - \sum_{i > j} \frac{\RR_{ij}}{\RR_{jj}} \SS_{:,i}, \quad
\SS_{:,j}^* = \frac{\mathbf{e}_j}{\RR_{jj}} - \sum_{i \in \Gamma} \frac{\RR_{ij}}{\RR_{jj}} \SS_{:,i},
\end{equation*}
where {$\Gamma$ denotes the set of indices of the most recently computed columns that are used in the approximation.}

The $\epsilon_{wd}$ error can be expressed as
\begin{equation}
\notag\epsilon_{wd} = \frac{\|\SS_{:,j} - \SS_{:,j}^*\|_2}{\|\SS_{:,j}\|_2} = \frac{\|\sum_{i \notin \Gamma} \frac{\RR_{ij}}{\RR_{jj}} \SS_{:,i}\|_2}{\|\SS_{:,j}\|_2}.
\end{equation}

While exact computation of $\epsilon_{wd}$ is challenging due to its dependence on the graph structure and algorithm's progression, our experiments demonstrate that the adaptive window size approach effectively balances error and efficiency. 
\end{proof}

The error bound in Eq.~\eqref{eq:bound_chol} captures the impact of sparse inverse estimation and the sliding window technique in $\chol$. Notably, $\|\SS_{:,u}\|_2^2$ approximates $(\calL_v^{-1})_{uu}$, with accuracy depending on the incomplete Cholesky factorization quality. The overall precision of our method in approximating $(\calL_v^{-1})_{uu}$ is thus influenced by both the error bound from Theorem~\ref{thm:complexity_and_error} and the incomplete Cholesky factorization's effectiveness. Experimental results in Section 6 demonstrate that $\chol$ maintains high accuracy in large-scale, real-world networks despite these cumulative approximations.

In terms of time complexity, $\chol$ demonstrates competitive performance compared to the algorithm in~\cite{ZhXuZh20}, which has a complexity of $\tilde{O}(m/\epsilon^2)$, with $\epsilon$ typically set small for accuracy. Our algorithm generally exhibits lower time complexity, particularly in large networks where $m$ approaches $n$. Additionally, the parameter $w_s$ allows for flexible management of the trade-off between computational complexity and accuracy.

\section{$\walk$ Algorithm: Sampling Spanning Trees}
\label{sec:alg2}

In the previous section, we proposed an algorithm based on incomplete Cholesky factorization and sparse inverse estimation but could not provide the final cumulative approximation error. Moreover, as will be discussed in Section~\ref{sec:exp}, $\chol$ has a non-negligible approximation error. Seeking a more accurate method, we note that the diagonal elements of $\calL_v^{-1}$ are related to \textit{random walk with a trap} on graphs~\cite{ZhQiZhXiGu09,ZhYa12}, where a random walk stops upon reaching the trapping node. Specifically, the $u$-th diagonal element of $\calL_v^{-1}$ represents the expected number of visits to $u$ by a random walker starting from $u$ before meeting the trapping node $v$.

\begin{lemma}\cite{ZhYa12}\label{lem:connection}
Given a network with node $v$ being the trap, let $\ttt\in \mathbb{R}^{(n-1)\times 1}$ denote the vector whose $u$-th element $\ttt_u$ represents the expected number of visits to node $u$ for a random walk starting from $u$ before meeting the trapping node $v$. Then $\ttt_u$ and $\calL_{v}^{-1}$ can be connected as
    \begin{align}\label{eq:connection}
        \ttt_u = (\calL_{v}^{-1})_{uu} , \text { for all } u \in V, u \neq v.
    \end{align}
\end{lemma}

Based on the above lemma, a straightforward approach to approximate $\ttt$ would be simulating multiple random walks from each node in the network with the trap node $v$ and counting the number of returns to that node. However, this method is essentially a repetitive process performed $n-1$ times to approximate individual diagonal elements, flagging potential rooms for improvement. Furthermore, in this method, we only count and use the number of random walk's returns to the initial node during the simulation. This approach discards information from other steps of the random walk, which could potentially be useful.

To address this issue, we observe an interesting connection between the random walk with a trap and the loop-erased random walk~\cite{La80, LaGF12}. The loop-erasure operation on a random walk involves removing loops in the order they occur. In a graph \(\calG\) and a random walk trajectory \(\gamma = (v_1, \cdots, v_l)\) on \(\calG\), we define the loop-erased trajectory \(\operatorname{LE}(\gamma) = (v_{i_1}, \cdots, v_{i_j})\) by eliminating loops. Specifically, \(i_j\) is determined iteratively: \(i_1 = 1\), and \(i_{j+1} = \max \{i \mid v_i = v_{i_j}\} + 1\). Loop-erased random walks avoid self-intersections and halt when revisiting the previous trajectory.

Wilson's algorithm, which applies loop-erasure to a random walk, was introduced to generate a uniform spanning tree rooted at a specified node~\cite{Wi96}. The algorithm involves four key steps: (\romannumeral1) initializing a tree with the root node, (\romannumeral2) fixing an arbitrary node ordering \(u_1, \cdots, u_{n-1}\), (\romannumeral3) following this order to perform an unbiased random walk until it reaches a node within the tree, and (\romannumeral4) applying the loop-erasure operation to add nodes and edges to the tree. This process continues until all nodes are included in the tree, ensuring uniform sampling of all possible spanning trees~\cite{Wi96}. Interestingly, when setting node \(v\) as the root, the expected number of visits on each node \(u\) equals \(\ttt_u\)~\cite{Wi96}, which is the expected number of visits on \(u\) for a random walk with a trap starting from \(u\). This connection can be formally expressed as follows.

\begin{lemma}\label{lem:visits}~\cite{Wi96}
    Suppose that we generate a spanning tree using Wilson's method with $v$ being the root. Let $\tilde{\ttt}\in \mathbb{R}^{(n-1)\times 1}$ be a vector of random variables with its $u$-th element being the random variable of the number of visits to a node $u$. Then, we have
    \begin{align}\label{eq:variable}
    \mean{\tilde{\ttt}_u}=\ttt_u=(\calL_v^{-1})_{uu}, \text { for all } u \in V, u \neq v .    
    \end{align}
\end{lemma}

Using this lemma, we can independently sample \( l \) spanning trees with Wilson's algorithm, treating the trapping node \( v \) as the root. In the $i$-th sample, we track the number of visits, \( \tilde{t}_{ui} \), to each node \( u \in V\setminus\{v\} \). Consequently, we can unbiasedly estimate \(\ttt_u\) as \(\tilde{\ttt}_{u} = \sum_{i=1}^{l} \tilde{t}_{ui}/l\), the average visits to node \( u \) across the \( l \) samples. Importantly, each step of the random walk yields valuable information. The main challenge is determining the number of samples needed to bound the approximate error, usually achieved using Hoeffding's Inequality.

\begin{theorem}
\label{c-h thm}
(\textsc{Hoeffding's Inequality}~\cite{hoeffding}). Let $Z_1, Z_2, \ldots, Z_{n_z}$ be i.i.d. random variables with $Z_j\left(\forall 1 \leq j \leq n_z\right)$ being strictly bounded by the interval $\left[a, b\right]$. We define the average of these variables by $Z= \frac{1}{n_z}\sum_{j=1}^{n_z} Z_j$. Then, for any $\eps >0 $, we have
$$
\mathbb{P}[|Z-\mathbb{E}[Z]| \geq \epsilon] \leq 2\exp \bigg(\frac{-2 \epsilon^2n_z}{\left(b-a\right)^2}\bigg).
$$
\end{theorem}
However, Theorem~\ref{c-h thm} requires a finite upper bound of the estimator, which does not exist for \( \tilde{t}_{ui}\). Instead, we bound \(\tilde{t}_{ui}\) with a high probability, as shown in the following lemma.

\begin{lemma}\label{lem:passnum-upperbound}
    Given a graph \(\calG\) and a trap node \(v\), if \(t\) is selected satisfying \(t\geq\log(m/(n\sqrt{n-1})\norm{\dd_{-v}}_2)/\log(\lambda)\), then
    \begin{equation*}
        \Pr{(\tilde{t}_{ui}>t)}\leq 1/{n},
    \end{equation*}
    where $\lambda$ denotes the largest absolute eigenvalue of matrix $\PP_{v}$.
\end{lemma}
\begin{proof}
Since \(\tilde{t }_{ui}\) denotes the number of visits of node \(u\) for the $i$-th sample, it is easy to verify that \(\tilde{t}_{ui}\leq T_{uv}/2\), where \(T_{uv}\) denotes the hitting time from \(u\) to \(v\) for this sample. According to Lemma~5.4 in~\cite{zhou2023opinion}, the proportion of samples of which node $u$ could not reach $v$ within $t$ steps is less than $1/n$, completing the proof.
\end{proof}
{
\begin{lemma}[Lemma 5.4 from~\cite{zhou2023opinion}]
Given a graph $\calG=(V,E)$, a node $v \in V$, if maximum length $t$ of random walks satisfies  $t= \log(m\beta /\sqrt{n-s}\norm{\dd_{-v}}_2)/\log(\lambda)$, where $\lambda$ is the spectral radius of $\PP_{-v}$, then the expected fraction of random walks that do not reach $v$ within $t$ steps is less than $\beta$. 
\end{lemma}
}

Next, we introduce a theoretical bound for the sample size \(l\).

\begin{lemma}\label{lem:trace-approx}
Given a connected graph \(\calG=(V,E)\), a trap node \(v\), and an error parameter \(\epsilon\), if we sample \(l\) independent random spanning trees using the Wilson's method 
 with $v$ being the root node, satisfying \(l\geq2\epsilon^{-2}\log 2n\log^2(m/(n\sqrt{n-1})\|\dd_{-v}\|_2)/\log^2(\lambda)\),
    \begin{align}\label{eq:abs_error}
    \Pr{\left(\Big|(\calL_v^{-1})_{uu} - \tilde{\ttt}_u\Big|\geq \frac{\eps}{2}\right)}\leq \frac{2}{n}.
    \end{align}
\end{lemma}
    \begin{proof}
        Lemma~\ref{lem:passnum-upperbound} reveals that \(\tilde{t}_{ui}\) exhibits an explicit upper bound \(t=\log(m/(n\sqrt{n-1})\norm{\dd_{-v}}_2)/\log(\lambda)\) with a high probability. Plugging this into Theorem~\ref{c-h thm}, we obtain
        \begin{align*}
        \Pr\spar{\abs{\tilde{\ttt}_{u}-\ttt_u}\geq\frac{\epsilon}{2}}
             & \leq 1-\Big({1-2\exp\Big(-\frac{2l\epsilon^2}{4t^2}\Big)}\Big)\Big({1-\frac{1}{n}}\Big)                                   \\
             & \leq 1-\Big({1-\frac{1}{n}}\Big)^2\leq1-\Big({1-\frac{2}{n}}\Big)=\frac{2}{n},
        \end{align*}
        where the second inequality follows from \(l\geq (2t^2\log 2n)/\eps^2\). This concludes the proof.
    \end{proof}

\begin{algorithm}[h]
	\caption{$\walk(\calG, v, \eps)$}\label{alg:fastdiag}
	\Input{
		A connected graph $\calG=(V,E)$; the pivot node $v$; the error parameter $\eps$
	}
	\Output{
		$\tilde{\calH}_u\ \forall u\in V$
	}
    $\theta \gets \frac{\eps}{2\Delta d_{\max} \left( 1 + \frac{d_{\max}(1 + n)}{2m} \right)^2\sqrt{mn\Delta}}$\;
    $l \gets 2\epsilon^{-2}\log 2n\log^2(m/(n\sqrt{n-1})\|\dd_{-v}\|_2)/\log^2(\lambda)$\;
    $\tilde{\ttt}  \gets \zero^{(n-1)\times 1}$\;
    Fix an arbitrary ordering $(1, \cdots, n-1)$ of $V\setminus\{v\}$\;
    \For{$i=1:l$}{
    $InTree[u] \leftarrow false, Next[u] \leftarrow -1 \ \forall u\in V$\;
    $InTree[v] \leftarrow true$\;
    \For{$j=1:n-1$}{
    $u \gets j, \tilde{\ttt}_u \gets \tilde{\ttt}_u+\frac{1}{l}$\;
    \While{!Intree[u]}{
    $Next[u] \leftarrow RandomNeighbor(u)$\;
    $u\leftarrow Next[u], \tilde{\ttt}_u\gets\tilde{\ttt}_u+\frac{1}{l}$\;
    }
    $\tilde{\ttt}_u \gets \tilde{\ttt}_u-\frac{1}{l}$\;
    $u \leftarrow j$\;
    \While{!InTree[u]}{
    $InTree[u] \gets true, u\leftarrow Next[u]$\;
    }
    }
    }
$\tilde{\calH}_u \gets \textsc{AppRWC}(\calG, v, \theta, \tilde{\ttt})$\;
\Return {$\tilde{\calH}_u\ \forall u\in V$ \;}    
\end{algorithm}
\subsection{The Sampling Algorithm}

Building on our analysis from above, we introduce $\walk$ for approximating RWC, outlined in Algorithm~\ref{alg:fastdiag}. It takes a graph $\calG$, a pivot node $v$, and an error parameter $\eps$ as input. The algorithm computes the error parameter $\theta$ for the Laplacian solver and determines the required sample size \( l \) based on Lemma~\ref{lem:trace-approx}. After initializing necessary data structures and establishing a fixed node order, $\walk$ samples \( l \) rooted spanning trees using loop-erased random walks, efficiently tracking the next node to ensure loop erasure. During each sample, it records the visits to each node. Finally, it computes the approximate diagonal elements $\tilde{\ttt}$ of $\calL_v^{-1}$ and uses the $\textsc{AppRWC}$ algorithm to determine the RWC for each node as per Eq.~\eqref{eq:newformular}. The following theorem summarizes the performance of $\walk$.
\begin{theorem}\label{the:performance_wilson}
    Algorithm~\ref{alg:fastdiag} runs in $O(l\times \trace{(\II - \PP_{v})^{-1}} + \tilde{O}(m))$, where $\trace{(\II - \PP_{v})^{-1}}$ represents the trace of the matrix $(\II - \PP_{v})^{-1},$  $l$ is the sample size, and $\tilde{O}(\cdot)$ hides $\text{poly}(\log n)$ factors. The output ensures that the estimated RWC $\tilde{\calH}_u\ \forall u\in V$ satisfies $\pi_u|\tilde{\calH}_u - \calH_u|\leq \eps.$ 
\end{theorem}
\begin{proof}
The time complexity of Algorithm~\ref{alg:fastdiag} primarily consists of two components. The first component involves executing the Laplacian solver once, which has a time complexity of $\tilde{O}(m)$. The second component is sampling the rooted spanning trees using Wilson's method. This is dominated by the number of visits to all nodes during the $l$ samplings, which corresponds to $l$ times the summation of the diagonal elements of the matrix $(\II - \PP_{v})^{-1}$. This summation is indeed equal to $l \times \trace{(\II - \PP_{v})^{-1}}$.
    
To align the relative error bound from the Laplacian solver with the absolute error in  Eq.~\eqref{eq:abs_error}, we first observe  that for $\xx = \LL^\dagger\ee_v$ and $\tilde{\xx} = \texttt{LapSolve}(\LL, \ee_v, \theta)$, the following inequalities hold for both vectors: $\sqrt{\lambda_2} \cdot\|\xx\|_{\infty} \leq\|\xx\|_{\LL} \leq \sqrt{4m} \cdot\|\xx\|_{\infty}$ , where $\lambda_2$ is the second smallest eigenvalue of $\LL$~\cite{AnEuPr20}. Additionally, it is known that $\lambda_2 \geq 4 /(n \cdot \Delta)$~\cite{DeNa07}, where $\Delta$ represents the graph diameter. {Defining  $\MM_l = \big(\II - \frac{1}{2m}\DD^{\frac{1}{2}}\one\one^\top\DD^{\frac{1}{2}}\big)\DD^{\frac{1}{2}}$ and $\MM_r = \DD^{\frac{1}{2}}\big(\II - \frac{1}{2m}\DD^{\frac{1}{2}}\one\one^\top\DD^{\frac{1}{2}}\big)$, we proceed as follows to calculate the infinity norm of the matrices $\MM_l$ and $\MM_r$. The diagonal elements of \( \MM_r \) are given by \( (\MM_r)_{ii} = \dd_i^{1/2} - \frac{1}{2m} \dd_i \dd_i^{1/2} \), and the off-diagonal elements are \( (\MM_r)_{ij} = - \frac{1}{2m} \dd_i \dd_j^{1/2} \) for \( i \neq j \). The infinity norm is the maximum row sum of the absolute values of the elements in \( \MM_r \). Thus, for each row \( i \), \[ \sum_{j=1}^n |(\MM_r)_{ij}| \leq |(\MM_r)_{ii}| + \sum_{\substack{j=1, j \neq i}}^n |(\MM_r)_{ij}| .\] By simplifying, we get \[ \sum_{j=1}^n |(\MM_r)_{ij}| \leq \dd_i^{1/2} \Big( 1 + \frac{\dd_i}{2m} \Big) + \frac{\dd_i}{2m} \sum_{\substack{j=1,\, j \neq i}}^n \dd_j^{1/2} .\] To obtain an upper bound, we note that \( \dd_i \leq d_{\max} \) and \( \dd_i^{1/2} \leq d_{\max}^{1/2} \), leading to \[ \sum_{j=1}^n |(\MM_r)_{ij}| \leq d_{\max}^{1/2} \Big( 1 + \frac{d_{\max}}{2m} \Big) + \frac{d_{\max}}{2m} \sum_{\substack{j=1,\, j \neq i}}^n d_{\max}^{1/2} .\] Recognizing that \( \sum_{\substack{j=1,\, j \neq i}}^n d_j^{1/2} \leq n d_{\max}^{1/2} \), we further simplify to \[ \sum_{j=1}^n |(\MM_r)_{ij}| \leq d_{\max}^{1/2} \Big( 1 + \frac{d_{\max}}{2m} \Big) + \frac{n d_{\max} d_{\max}^{1/2}}{2m} .\] Therefore, the infinity norm is bounded by \[ \|\MM_r\|_\infty \leq d_{\max}^{1/2} \Big( 1 + \frac{d_{\max}(1 + n)}{2m} \Big) .\] Using similar methods, we can derive the same infinity norm for matrix $\MM_l$.}

Building on the above results, we have 
{
\begin{align}
&\quad \|\calL[:,v]-\tilde{\yy}\|_{\infty} \leq \|\MM_l\|_\infty \|\MM_r\|_\infty \|\tilde{\xx}-\xx\|_{\infty}\notag \\ \notag&\leq\frac{\|\MM_l\|_\infty^2}{\sqrt{\lambda_2}} \|\tilde{\mathbf{x}}-\mathbf{x}\|_{\LL} 
\leq \frac{\theta \|\MM_l\|_\infty^2 }{\sqrt{\lambda_2}} \|\mathbf{x}\|_{\LL} \\
&\leq \frac{\theta \|\MM_l\|_\infty^2 }{\sqrt{4 /(n \Delta)}} \|\mathbf{x}\|_{\LL}\leq \frac{2\theta \|\MM_l\|_\infty^2 \sqrt{mn \Delta}}{2} \|\mathbf{x}\|_{\infty} \notag\\
&\leq \theta d_{\max} \Big( 1 + \frac{d_{\max}(1 + n)}{2m} \Big)^2 \Delta\sqrt{mn\Delta}\leq \frac{\eps}{2}.\label{eq:infinity_norm}
\end{align}
}
The first inequality arises from basic linear algebra. 
The penultimate inequality comes from the fact that $\xx$ expresses potentials scaled by $1/n$, arising from $n-1$ effective resistance problems fused together~\cite{Te91}. The maximum norm of $\xx$ can thus be bounded by $(n-1)\frac{1}{n}R_{uv} \leq \Delta$, because the graph diameter limits the effective resistance. The last inequality follows from $$\theta = \frac{\eps}{2\Delta d_{\max} \left( 1 + \frac{d_{\max}(1 + n)}{2m} \right)^2\sqrt{mn\Delta}}.$$
By combining Eq.~\eqref{eq:abs_error} and Eq.~\eqref{eq:infinity_norm}, and considering that the pivot node is chosen with a much larger degree than other nodes, we conclude the proof.
\end{proof}

{
\noindent\textbf{Discussion: relation of \textsc{FastChol} and \textsc{FastWalk}.}
Both algorithms arise from the pivot-based reformulation but estimate the same diagonal quantity of $\mathcal{L}_v^{-1}$ via \emph{orthogonal} primitives: \textsc{FastChol} leverages the positive definiteness of $\mathcal{L}_v$ to build a sparse triangular factor and read diagonals by cheap solves, whereas \textsc{FastWalk} relies on the loop-erased walks to form a sample-based estimator.  We therefore view them as complementary choices.}

\section{Experimental Evaluation}
\label{sec:exp}
\subsection{Experimental Setting}
\textbf{Datasets.} To evaluate the efficiency and accuracy of our proposed algorithms, namely $\chol$ and $\walk$, we conduct experiments on 6 real world networks from $\mathrm{SNAP}$~\cite{LeSo16}~(Table~\ref{tab:statistics}), {which span several orders of magnitude in size and a wide range of sparsity—measured by average degree.}

\begin{table}[h]
  \caption{Basic statistics of the six studied real-world networks. Dia. and AvgDeg. represent the graph diameter and the average degree, respectively.}
  \label{tab:statistics}
  \setlength{\tabcolsep}{2.4pt}
  \fontsize{9}{8}\selectfont
  \begin{tabular}{lrrcc}
    \toprule
    Network & \#nodes ($n$) & \#edges ($m$) & Dia. ($\Delta$) & AvgDeg. \\
    \midrule
    Facebook            &      4,039 &       88,234 &  8  & 43.7 \\
    DBLP                &    317,080 &    1,049,866 & 23  &  6.6 \\
    Youtube             &  1,134,890 &    2,987,624 & 24  &  5.3 \\
    Orkut               &  3,072,441 &  117,185,083 & 10  & 76.3 \\
    soc-LiveJournal     &  3,997,962 &   34,681,189 & 16  & 17.3 \\
    hetero-LiveJournal  & 12,678,882 &  160,995,330 & 17  & 25.4 \\
    \bottomrule
  \end{tabular}
\end{table}

\noindent \textbf{Implementation Details.} All experiments are run on a Linux machine with an Intel Xeon(R) Gold 6240@2.60GHz 32-core processor and 256GB of RAM. Before timing, the graph data is loaded into memory. While most algorithms are implemented in Julia, the incomplete Cholesky factorization in $\chol$ is implemented in MATLAB using the \texttt{ichol} function. Results are excluded if the algorithm does not return within 24 hours. For $\chol$ and $\walk$, the node with the highest degree is selected as the pivot node (see Eq.~\eqref{eq:newformular}), which reduces errors for $\chol$ and improves efficiency for $\walk$. The largest absolute eigenvalue $\lambda$ of $\PP_{v}$ for each network is approximated via ARPACK~\cite{lehoucq1998arpack}. In $\chol$, the drop tolerance $\delta$ is set to $10^{-4}$, 
$\zeta$ is set to $O(\log n)$, and the window size is generally $O(\log n)$.
For fairness, the same $\theta$ value is used in $\chol$ as in $\walk$. Experiments are conducted on the largest connected component of each network using a single thread. 
\noindent \textbf{Competitors and Ground Truth.} To showcase the superiority of our proposed algorithms, we compare them against the $\solver$ algorithm from~\cite{ZhXuZh20}. We compute the exact RWC as the ground truth for the Facebook dataset. For larger graphs, it is not computationally possible to obtain the exact value of RWC using the existing approaches. Therefore, considering the remarkable accuracy of $\walk$ as will be revealed in Fig.~\ref{fig:accuracy} (a), we treat the output of $\walk$ with $\eps = 0.0001$ as the ground truth for RWC.

\noindent \textbf{Evaluation Measures.} To evaluate the approximation \textit{quality}, we measure the mean relative errors and the Kendall tau distance~\cite{kendal38}: a metric used to quantify the degree of disagreement between the rankings produced by different methods. { It is defined mathematically as:
$
K\left(L_1, L_2\right) = \frac{2(\alpha - \beta)}{n(n - 1)}.
$
In this formula, \(\alpha\) represents the number of concordant pairs, while \(\beta\) denotes the number of discordant pairs. The value of the distance ranges from \(-1\) to \(1\), where a value of \(1\) indicates complete agreement between the two rankings, and a value of \(-1\) signifies total disagreement.} We use the wall-clock running time to evaluate the \textit{efficiency} of algorithms. 
\subsection{Efficiency}
We first investigate the efficiency of our algorithms $\chol$ and $\walk$. To this end, in Fig.~\ref{fig:efficiency}, we report the running time of $\chol$, $\walk$ and that of the $\solver$ algorithm from~\cite{ZhXuZh20} when $\eps$ is varied from 0.01 to 0.3.  The results show that for all three algorithms, the execution time decreases when $\eps$ increases, implying that there is a trade-off between running time and accuracy. Moreover, we can see that for all approximation parameters $\epsilon$, the computational time for $\chol$ and $\walk$ is significantly smaller than that of $\solver$, especially for large-scale networks, which is consistent with our analysis in Sections~\ref{sec:alg1} and~\ref{sec:alg2}. Thus, $\chol$ and $\walk$ can significantly improve the efficiency compared to $\solver$. For the networks with more than $3$ million nodes (Orkut, soc-Livejournal and hetero-Livejournal), the $\solver$ algorithm runs out of memory. However, for these networks, we can approximately compute the RWC for all nodes using algorithms $\chol$ and $\walk$, further showcasing the efficiency and scalability of our proposed methods.

Regarding the running time comparison between $\chol$ and $\walk$, Fig.~\ref{fig:efficiency} demonstrates that $\chol$ outperforms $\walk$ in most cases. This advantage is likely attributed to the use of sparsification and sliding window strategies in $\chol$, whereas in dense graphs, the random walks in $\walk$ may take too many steps before getting trapped.
\begin{figure}
    \centering
    \includegraphics[width=\columnwidth]{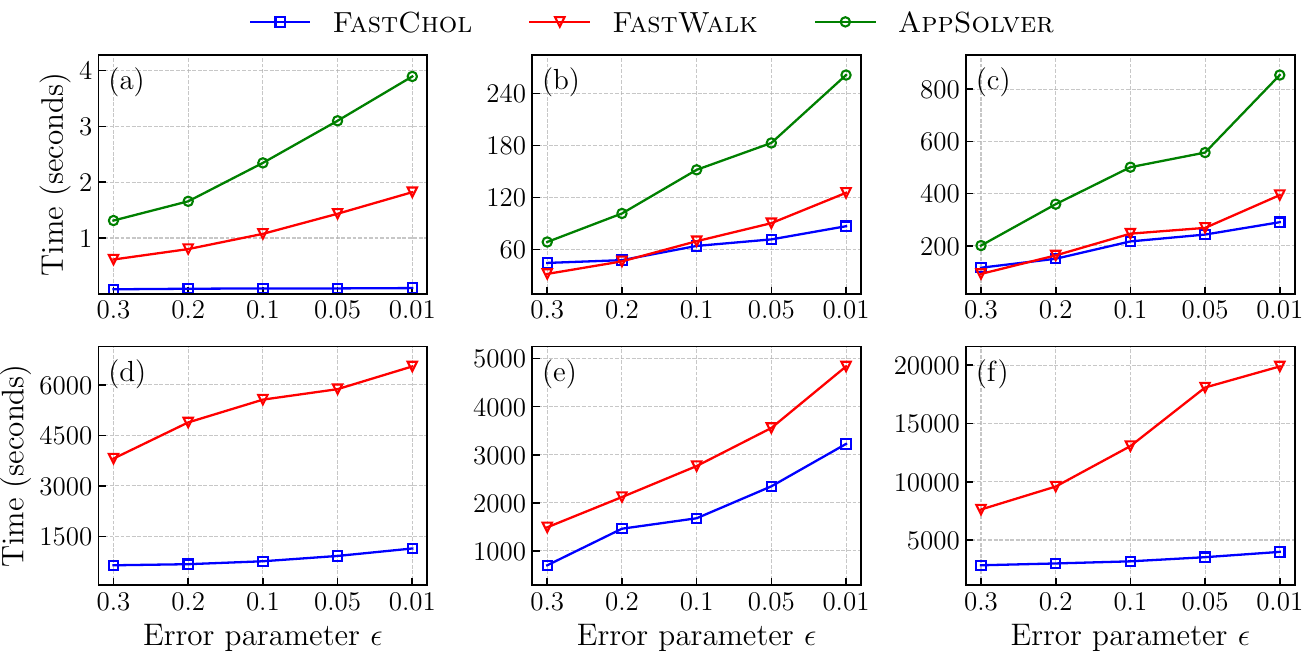}
    \vspace{-0.45cm}
    \caption{Running time comparison for different algorithms on several networks: (a) Facebook, (b) DBLP, (c) Youtube, (d) Orkut, (e) soc-Livejournal and (f) hetero-Livejournal.}
    \label{fig:efficiency}
\end{figure}
\begin{figure}
    \centering
    \includegraphics[width=\columnwidth]{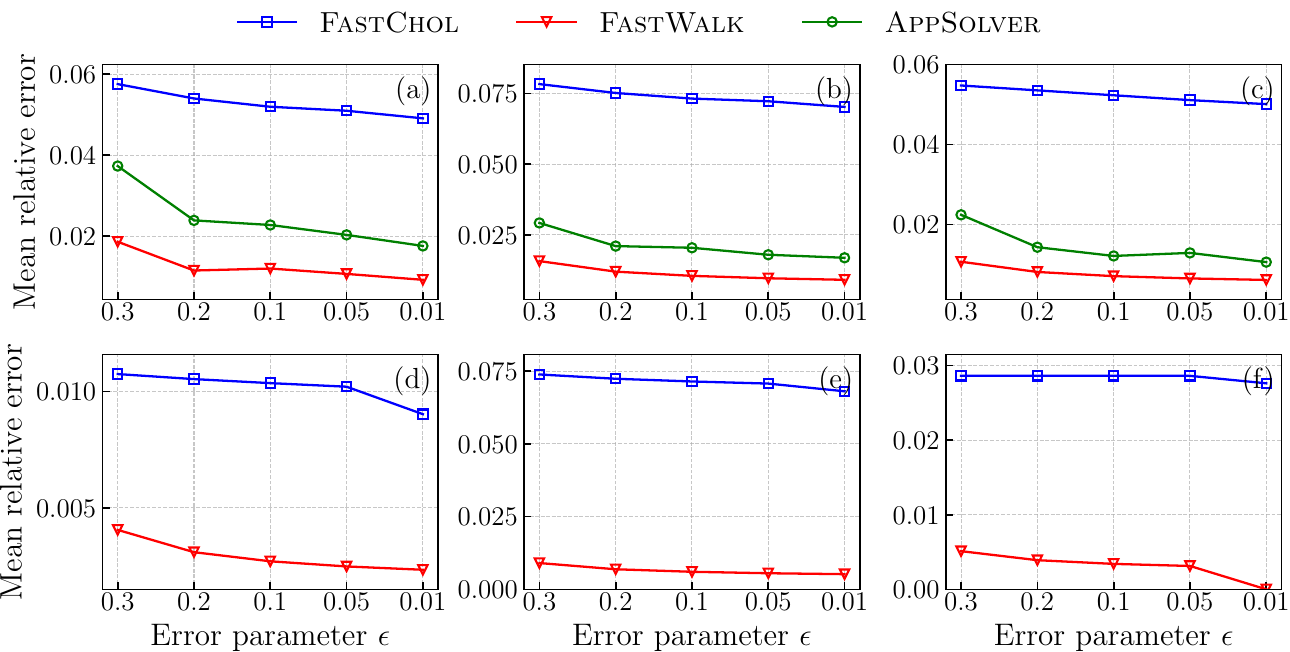}
    \vspace{-0.45cm}
    \caption{Mean relative error for different algorithms on several networks: (a) Facebook, (b) DBLP, (c) Youtube, (d) Orkut, (e) soc-Livejournal and (f) hetero-Livejournal.}
    \label{fig:accuracy}
\end{figure}
\subsection{Approximation Quality}\label{sec:accuracy}
In addition to high efficiency, our algorithms also provide a great level of accuracy. To demonstrate this, we compare the results of $\walk$ and $\chol$ with those of $\solver$. In Fig.~\ref{fig:accuracy}, we report the mean relative errors of different algorithms, which is defined as $\frac{1}{n}\sum_{u\in V}|{\calH_u}-\tilde{\calH}_u|/{\calH_u}$. Based on Fig.~\ref{fig:accuracy}, generally for all three algorithms and choices of $\epsilon$, the error value is negligible. $\walk$ provides the best accuracy among the three algorithms. While $\chol$'s error is slightly larger than $\walk$ and $\solver$, it is significantly faster, as we discussed in the previous subsection. Thus, we conclude that the incomplete Cholesky factorization and the sparse inverse estimation techniques enhance the efficiency, with a minor compromise on the accuracy.

\begin{figure}
    \centering
    \includegraphics[width=0.8\columnwidth]{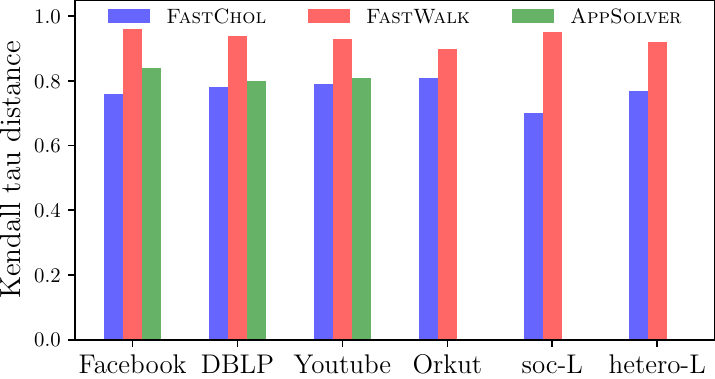}
    \vspace{-0.45cm}
    \caption{Kendall tau distance for different algorithms.}
    \vspace{-0.45cm}
\label{fig:kendall}
\end{figure}

In many cases, nodes are ranked based on their centrality values. Therefore, we use the Kendall tau distance to measure the rank correlation between the approximated RWC and the ground truth values.
We report in Fig.~\ref{fig:kendall} the ranking correlation of different algorithms with $\epsilon=0.1$. The results show that algorithm $\walk$ always obtains higher Kendall tau distances than $\solver$, while $\chol$ obtains slightly lower Kendall tau distance than $\solver$. However, this rank quality sacrifice is acceptable considering the distinguished efficiency performance of $\chol$ as shown in Fig.~\ref{fig:efficiency}.

\subsection{Memory Overhead}
{\setlength{\tabcolsep}{4.5pt}
\begin{table}[t]
  \centering
  \caption{Peak memory usage (GB) comparison over six real-world graphs.}
  \label{tab:mem}
  \renewcommand{\arraystretch}{1.12}
  \begin{tabular}{lccc}
    \toprule
    & \solver & \chol & \walk \\
    \midrule
     Facebook              & 3.781 & 0.575 & 0.729 \\
     DBLP                  & 8.312 & 0.746 & 1.096 \\
     Youtube               & 53.829 & 1.526 & 2.356 \\
     Orkut                 & -- & 26.494 & 47.825 \\
     soc-LiveJournal       & -- & 26.357 & 14.823 \\
     hetero-LiveJournal    & -- & 11.496 & 67.561 \\
    \bottomrule 
  \end{tabular}
\end{table}}

{To gauge the memory overhead of the three algorithms, \solver, \chol, and \walk, we recorded the peak-resident memory each requires on six publicly available graphs (Table~\ref{tab:mem}).  Missing entries indicate that the run terminated because available RAM was exhausted.

The measurements reveal a clear hierarchy. On the moderate-scale Facebook, DBLP, and Youtube graphs $\solver$ consumes 3.8 GB, 8.3 GB, and 53.8 GB, respectively, while $\chol$ holds the footprint to 0.6–1.5 GB and $\walk$ to 0.7–2.4 GB—up to a 35-fold reduction relative to the baseline.
As graph size or density increases, the gap widens. $\solver$ fails outright on Orkut, soc-LiveJournal, and hetero-LiveJournal, yet the other two methods still complete. On dense Orkut $\chol$ peaks at 26.5 GB, roughly half of $\walk$’s 47.8 GB; on the sparser soc-LiveJournal $\walk$ drops to 14.8 GB while $\chol$ rises to 26.4 GB; on hetero-LiveJournal $\chol$ again leads with 11.5 GB, whereas $\walk$ climbs to 67.6 GB. Overall, $\chol$ delivers the most consistent, and typically the lowest, memory footprint, extending the solvable range well beyond that of the solver, while $\walk$ provides a competitive alternative whose footprint varies with graph structure.
}

\subsection{Ablation Study}
{ To better understand the source of $\chol$'s performance, we now conduct a detailed ablation study. This analysis focuses on $\chol$ due to its multi-component design, whereas $\walk$'s behavior is more directly governed by the sample size $l$.
We evaluate the full algorithm against three variants to isolate the contribution of our key optimization on the DBLP dataset, disabling the sparsification step (\textsc{NoSparsify}), removing the sliding window while utilizing the full available columns (\textsc{NoWindow}), and using a fixed window (\textsc{NoAdaptive}).
As shown in Fig.~\ref{fig:ablation}, the results highlight the distinct role of each optimization. Note that the y-axis for time in subplot (a) is on a logarithmic scale to better visualize the differences.

The sliding window is essential for efficiency: removing it in \textsc{NoWindow} makes runtime rise sharply as $\epsilon$ decreases and yields the slowest curve. With windowing retained, the running time of $\chol$ and \textsc{NoSparsify} stays nearly flat in $\epsilon$; \textsc{NoSparsify} is consistently slower than $\chol$ due to denser operators, and \textsc{NoAdaptive} offers no clear runtime advantage over $\chol$. On accuracy, \textsc{NoAdaptive} has the largest errors across $\epsilon$, demonstrating the importance of our adaptive strategy for maintaining accuracy. Among the other three, errors are close— $\chol$ is slightly less accurate than \textsc{NoSparsify}, which is slightly less accurate than \textsc{NoWindow}—in exchange for $\chol$’s pronounced speedup.
}

\begin{figure}
    \centering
\includegraphics[width=\columnwidth]{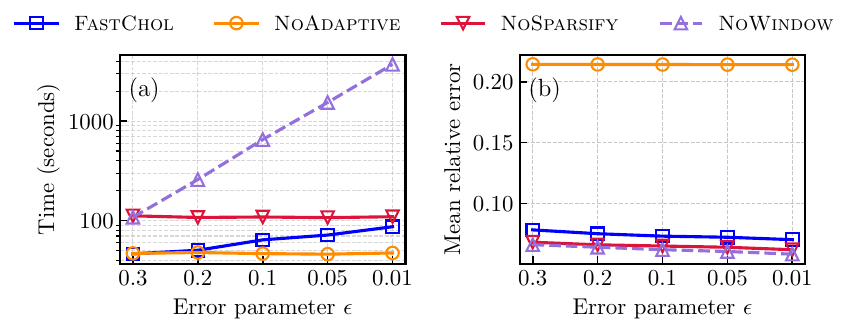}
    \vspace{-0.45cm}
    \caption{Ablation study for $\chol$ on the effects of key components. (a) time comparison and (b) error comparison.}
    \vspace{-0.45cm}
\label{fig:ablation}
\end{figure}

\subsection{Hyperparameter Sensitivity}
{ To provide guidance for parameter tuning, we investigate the sensitivity of $\chol$'s performance to its key hyperparameters: the base window size ($w_s$) and the \textsc{IChol} drop tolerance ($\delta$). The analysis is conducted on the DBLP dataset by varying one parameter at a time. The impact on running time and relative error is reported in Fig.~\ref{fig:hyper}.

As shown in Fig.~\ref{fig:hyper} (a), we observe a clear trade-off associated with the base window size. Specifically, increasing $w_s$ leads to a significant reduction in mean relative error, but at the cost of a corresponding increase in computation time. Similarly, Fig.~\ref{fig:hyper} (b) illustrates that a smaller drop tolerance $\delta$ enhances accuracy but also increases the running time. Extremely small values for $\delta$ can lead to out-of-memory errors. These results highlight the importance of balancing these parameters to achieve a desired trade-off between computational efficiency and approximation quality.}

\begin{figure}
    \centering
\includegraphics[width=\columnwidth]{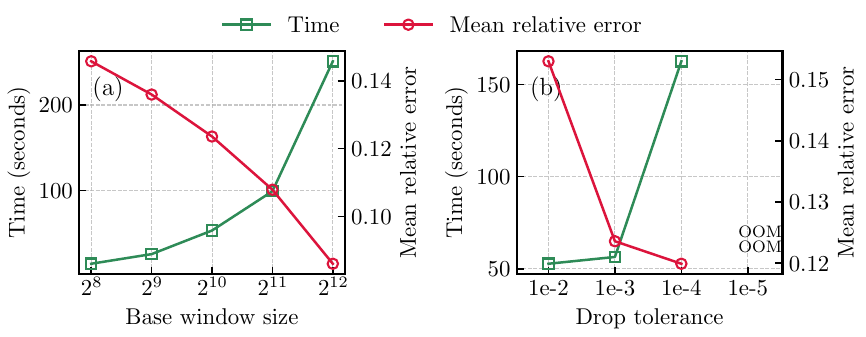}
    \vspace{-0.45cm}
    \caption{Sensitivity of hyperparameters: (a) base window size and (b) drop tolerance.}
    \vspace{-0.45cm}
\label{fig:hyper}
\end{figure}

\section{Related work}
Identifying ``crucial'' nodes is a fundamental problem in network science and graph data mining~\cite{LaMe12,BaZh22}, essential for network analysis with applications across various fields~\cite{Ne10,ZhYaTa23,PeLe23,CoCyWo21,deAlda20}. Over the decades, many centrality measures have been introduced to characterize nodes' roles in networks~\cite{WhSm03,BoVi14,BeKl15,BoDeRi16,LU20161}. Among various centrality indices, betweenness centrality and closeness centrality are arguably the two most frequently used ones, especially in social network analysis~\cite{BaAl48,Ba50}. However, both metrics only consider the shortest paths, neglecting contributions from other paths, which can lead to counterintuitive results~\cite{BeWeLuMe16}. To address this, measures accounting for all paths between nodes have been proposed, such as PageRank~\cite{Chung2010PageRankAR,PRbeyond}, Katz centrality~\cite{Katz1953}, and information centrality~\cite{Ne05,stephenson1989rethinking}. This paper focuses on random walk centrality~\cite{NoJaRi04,lalo93,MaChMa15}, a powerful centrality measure which is capable of capturing complex global structural information leveraging information from all paths.

Due to the extensive applications of random walk centrality across domains~\cite{LoMaKa07, BlFlTh11,JoBrKi19,OlStFu19,RiAlBo20}, its computation has garnered significant attention. Exact computation requires cubic time relative to the number of nodes, which is impractical for large networks. The state-of-the-art method proposed by Zhang et al.~\cite{ZhXuZh20} expresses random walk centrality in terms of quadratic forms of the Laplacian pseudo-inverse and develops a fast algorithm based on the Johnson-Lindenstrauss lemma and Laplacian solver. Nevertheless, the substantial memory requirements of this Laplacian solver limit its practicality for large-scale networks comprising millions of edges. Additionally, their approach requires {$O(\log n/\epsilon^2)$} calls to the Laplacian solver for an error parameter $\eps$, making it computationally expensive.

In this paper, we introduced a fast algorithm leveraging the positive definiteness of the normalized Laplacian submatrix, utilizing incomplete Cholesky factorization and the sparse inverse of the factorization components. Additionally, we proposed a novel algorithm based on spanning tree sampling, derived from the relation between the inverse of the normalized Laplacian submatrix and random walks. While matrix factorization and spanning tree sampling techniques are widely used in designing various graph algorithms~\cite{Liao2023EfficientRD,wang2017fora,LiAn2023,DaTiRa16,BeMi02,HePhSo20,XiXuZh25,InAuOh24,LiHuLe21,XiZh24}, our approach uniquely addresses the computation of RWC in a non-trivial manner.

\section{Conclusion}

{In this paper, we presented a new pivot-based formulation for computing Random Walk Centrality (RWC). This formulation reduces the core computation to a \emph{single} Laplacian system solve and a subsequent diagonal block estimation, significantly lowering the problem's theoretical complexity.  From this formulation, we developed $\chol$ and $\walk$ by pioneering the use of incomplete Cholesky factorization and loop-erased random walks for large-scale RWC computation.
Experiments on real-world networks with up to 160 million edges show our algorithms' superior performance, delivering up to 35$\times$ memory and 10$\times$ time savings over the state-of-the-art method while maintaining comparable accuracy. Notably, our algorithms succeed even when the baseline method fails due to memory limitations. These achievements make full-graph RWC analysis practical on common commercial hardware, opening new doors for research previously limited by network scale. 

Future work could enhance the current algorithms, for instance by improving the accuracy of the sparse Cholesky approximation and devising more efficient sampling techniques. Beyond this, we will also explore a lightweight hybrid that adaptively balances the two methods
(see Section~\ref{sec:alg2} for a brief discussion). Extending the framework is also a promising direction, including adapting the methods for real-time computation on dynamic graphs and generalizing them to calculate RWC for groups of nodes.}

\providecommand{\noopsort}[1]{}\providecommand{\singleletter}[1]{#1}

\vfill
\pagebreak 

\begin{IEEEbiography}
{Changan Liu} received the B.Eng. degree from the School of Software, Dalian University of Technology, Dalian, China, in 2019. He received the Ph.D. degree from the School of Computer Science, Fudan University, Shanghai, China, in 2024. He is currently a Research Fellow at Nanyang Technological University, Singapore. His research interests include network science, computational social science, graph data mining, social network analysis, and graph foundation models.
\end{IEEEbiography}
\begin{IEEEbiography}
{Zixuan Xie}
received the B.Eng. degree in software engineering from Fudan University, Shanghai, China, in 2024. She is currently pursuing the Ph.D. degree in Department of Computer Science, University of Virginia, Charlottesville, United States. Her research interests include reinforcement learning theory, stochastic approximation, and in-context learning.
\end{IEEEbiography}
\begin{IEEEbiography}
{Ahad N. Zehmakan}
received his PhD degree in 2020 from ETH Zurich and is currently a faculty member in the School of Computing at the Australian National University. His research interests include graph algorithms, social network analysis, and graph neural networks.
\end{IEEEbiography}

\begin{IEEEbiography} 
    {Zhongzhi Zhang}
    (M'19)	 received the B.Sc. degree in applied mathematics from Anhui University, Hefei, China, in 1997 and the Ph.D. degree in management science and engineering from Dalian University of Technology, Dalian, China, in 2006. \\
    From 2006 to 2008, he was a Post-Doctoral Research Fellow with Fudan University, Shanghai, China, where he is currently a Full Professor with the College of Computer Science and Artificial Intelligence. He has published over 200 papers in international journals or conferences. 
    Since 2019, he has been selected as one of the most cited Chinese researchers 	(Elsevier) every year. 
    His current research interests include network science, graph data mining, social network analysis, computational social science, spectral graph theory, and random walks. \\
    Dr. Zhang was a recipient of the Excellent Doctoral Dissertation Award of Liaoning Province, China, in 2007, the Excellent Post-Doctor Award of Fudan University in 2008, the Shanghai Natural Science Award (third class) in 2013, the Wilkes Award for the best paper published in The Computer Journal in 2019, and the CCF Natural Science Award (second class) in 2022. 
\end{IEEEbiography}

\end{document}